\documentclass{article}

\usepackage[utf8]{inputenc} 
\usepackage[T1]{fontenc}    
\usepackage{hyperref}       
\usepackage{url}            
\usepackage{booktabs}       
\usepackage{amsfonts}       
\usepackage{amsmath}        
\usepackage{amsthm}         
\usepackage{graphicx}       
\usepackage{nicefrac}       
\usepackage{microtype}      
\usepackage{tikz}           
\usepackage{subcaption}     
\usepackage{xcolor}         
\usepackage{algorithm}      
\usepackage{algpseudocode}  
\usepackage{algorithmicx}   
\usepackage{bm}             
\usepackage{url}            
\usepackage{iclr2017_conference, times}  
\iclrfinalcopy

\usetikzlibrary{shapes,positioning}
\tikzset{>=latex}

\newtheorem{prop}{Proposition}

\newcommand{\EE}{\mathbb{E}}
\newenvironment{sketchproof}{%
  \proof}{\endproof}
\newcommand{\figureHeight}{5cm}

\title{Adversarially Learned Inference}

\author{
  Vincent Dumoulin$^{1}$, Ishmael Belghazi$^{1}$, Ben Poole$^{2}$\\
  {\bf Olivier Mastropietro$^{1}$}, {\bf Alex Lamb$^{1}$}, {\bf Martin Arjovsky$^{3}$}\\
  {\bf Aaron Courville$^{1\dagger}$}\\
  $^1$ MILA, Universit\'e de Montr\'eal, \texttt{firstname.lastname@umontreal.ca}.\\
  $^2$ Neural Dynamics and Computation Lab, Stanford, \texttt{poole@cs.stanford.edu}. \\
  $^3$ New York University, \texttt{martinarjovsky@gmail.com}. \\
  $^\dagger$CIFAR Fellow.\\
}

\begin{document}

\maketitle

\begin{abstract}
We introduce the adversarially learned inference (ALI) model, which jointly
learns a generation network and an inference network using an adversarial
process. The generation network maps samples from stochastic latent variables to
the data space while the inference network maps training examples in data space
to the space of latent variables. An adversarial game is cast between these two
networks and a discriminative network is trained to distinguish between
joint latent/data-space samples from the generative network and joint samples
from the inference network.  We illustrate the ability of the model to learn
mutually coherent inference and generation networks through the inspections of
model samples and reconstructions and confirm the usefulness of the learned
representations by obtaining a performance competitive with state-of-the-art
on the semi-supervised SVHN and CIFAR10 tasks.
\end{abstract}

\section{Introduction}

Deep directed generative model has emerged as a powerful framework for modeling
complex high-dimensional datasets. These models permit fast ancestral sampling,
but are often challenging to learn due to the complexities of inference.
Recently, three classes of algorithms have emerged as effective for learning
deep directed generative models: 1) techniques based on the Variational
Autoencoder (VAE) that aim to improve the quality and efficiency of inference by
learning an inference machine~\citep{kingma2013auto,rezende2014stochastic}, 2)
techniques based on Generative Adversarial Networks (GANs) that bypass inference
altogether~\citep{goodfellow2014generative} and 3) autoregressive
approaches~\citep{van2016pixel,van2016pixelcnn,van2016wavenet} that forego
latent representations and instead model the relationship between input
variables directly. While all techniques are provably consistent given infinite
capacity and data, in practice they learn very different kinds of generative
models on typical datasets.

VAE-based techniques learn an approximate inference mechanism that allows reuse
for various auxiliary tasks, such as semi-supervised learning or inpainting.
They do however suffer from a well-recognized issue of the maximum likelihood
training paradigm when combined with a conditional independence assumption on
the output given the latent variables: they tend to distribute probability mass
diffusely over the data space~\citep{Theis2015}. The direct consequence of this
is that image samples from VAE-trained models tend to be
blurry~\citep{goodfellow2014generative,larsen2015autoencoding}. Autoregressive
models produce outstanding samples but do so at the cost of slow sampling speed
and foregoing the learning of an abstract representation of the data. GAN-based
approaches represent a good compromise: they learn a generative model that
produces higher-quality samples than the best VAE
techniques~\citep{radford2015unsupervised,larsen2015autoencoding} without
sacrificing sampling speed and also make use of a latent representation in the
generation process. However, GANs lack an efficient inference mechanism, which
prevents them from reasoning about data at an abstract level. For instance, GANs
don't allow the sort of neural photo manipulations showcased
in~\citep{brock2016neural}. Recently, efforts have aimed to bridge the gap
between VAEs and GANs, to learn generative models with higher-quality samples
while learning an efficient inference
network~\citep{larsen2015autoencoding,lamb2016discriminative,
dosovitskiy2016generating}. While this is certainly a promising research
direction, VAE-GAN hybrids tend to manifest a compromise of the strengths and
weaknesses of both approaches.

In this paper, we propose a novel approach  to integrate efficient inference
within the GAN framework. Our approach, called Adversarially Learned Inference
(ALI), casts the learning of both an inference machine (or encoder) and a deep
directed generative model (or decoder) in an GAN-like adversarial framework. A
discriminator is trained to discriminate joint samples of the data and the
corresponding latent variable from the encoder (or approximate posterior) from
joint samples from the decoder while in opposition, the encoder and the decoder
are trained together to fool the discriminator. Not only are we asking the
discriminator to distinguish synthetic samples from real data, but we are
requiring it to distinguish between two joint distributions over the data space
and the latent variables.

With experiments on the Street View House Numbers (SVHN) dataset
\citep{netzer2011reading}, the CIFAR-10 object recognition dataset
\citep{krizhevsky2009learning}, the CelebA face dataset~\citep{liu2015deep} and
a downsampled version of the ImageNet dataset~\citep{russakovsky2015imagenet},
we show qualitatively that we maintain the high sample fidelity associated with
the GAN framework, while gaining the ability to perform efficient inference. We
show that the learned representation is useful for auxiliary tasks by achieving
results competitive with the state-of-the-art on the semi-supervised SVHN and
CIFAR10 tasks.

\section{Adversarially learned inference}

\begin{figure}[t]
    \centering
    \begin{tikzpicture}[remember picture,node distance=2cm,
                        box/.style={draw,rectangle,rounded corners}]
        \node[box,rectangle] (q) {
            \begin{tikzpicture}
                \node (x) {$\bm{x} \sim q(\bm{x})$};
                \node[above=of x]
                    (z_hat) {$\hat{\bm{z}} \sim q(\bm{z} \mid \bm{x})$};
            \end{tikzpicture}
        };
        \node[box,minimum height=1cm,minimum width=2cm,right=of q]
            (discriminator) {$D(\bm{x}, \bm{z})$};
        \node[box,right= of discriminator] (p) {
            \begin{tikzpicture}
                \node (x_tilde) {$\tilde{\bm{x}} \sim p(\bm{x} \mid \bm{z})$};
                \node[above=of x_tilde] (z) {$\bm{z} \sim p(\bm{z})$};
            \end{tikzpicture}
        };
        \draw[->] (x) -- (z_hat) node[midway,above,rotate=90]
			{$\scriptstyle G_z(\bm{x})$};
        \draw[->] (z) -- (x_tilde) node[midway,above,rotate=270]
			{$\scriptstyle G_x(\bm{z})$};
        \draw[->] (q) -- (discriminator) node[midway,above]
			{$(\bm{x}, \hat{\bm{z}})$};
        \draw[->] (p) -- (discriminator) node[midway,above]
			{$(\tilde{\bm{x}}, \bm{z})$};
    \end{tikzpicture}
    \caption{\label{fig:model} The adversarially learned inference (ALI) game.}
\end{figure}
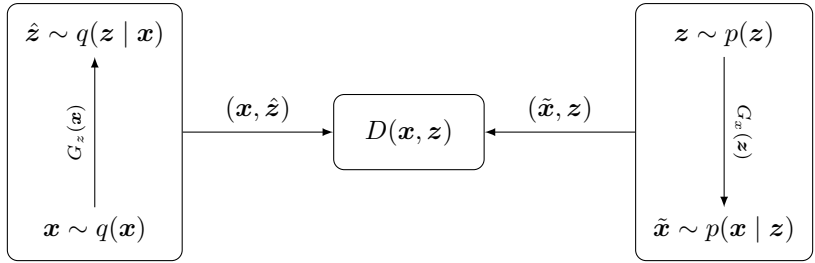

Consider the two following probability distributions over $\bm{x}$ and $\bm{z}$:
\begin{itemize}
    \item the \emph{encoder} joint distribution
		$q(\bm{x}, \bm{z}) = q(\bm{x}) q(\bm{z} \mid \bm{x})$,
    \item the \emph{decoder} joint distribution
		$p(\bm{x}, \bm{z}) = p(\bm{z}) p(\bm{x} \mid \bm{z})$.
\end{itemize}
These two distributions have marginals that are known to us: the encoder
marginal $q(\bm{x})$ is the empirical data distribution and the decoder
marginal $p(\bm{z})$ is usually defined to be a simple, factorized
distribution, such as the standard Normal distribution $p(\bm{z}) =
\mathcal{N}(\bm{0},\bm{I})$. As such, the generative process between $q(\bm{x},
\bm{z})$ and $p(\bm{x}, \bm{z})$ is reversed.

ALI's objective is to match the two joint distributions. If this is achieved,
then we are ensured that all marginals match and all conditional distributions
also match. In particular, we are assured that the conditional $q(\bm{z} \mid
\bm{x})$ matches the posterior $p(\bm{z} \mid \bm{x})$.

In order to match the joint distributions, an adversarial game is played. Joint
pairs $(\bm{x}, \bm{z})$ are drawn either from $q(\bm{x}, \bm{z})$ or
$p(\bm{x}, \bm{z})$, and a discriminator network learns to discriminate between
the two, while the encoder and decoder networks are trained to fool the
discriminator.

The value function describing the game is given by:
\begin{equation}
\label{eq:value_function}
\begin{split}
    \min_G \max_D V(D, G)
	&= \mathbb{E}_{q(\bm{x})}[\log(D(\bm{x}, G_z(\bm{x})))]
	 + \mathbb{E}_{p(\bm{z})}[\log(1 - D(G_x(\bm{z}), \bm{z}))] \\
    &= \iint q(\bm{x}) q(\bm{z} \mid \bm{x})
		     \log(D(\bm{x}, \bm{z})) d\bm{x} d\bm{z} \\
    &+ \iint p(\bm{z}) p(\bm{x} \mid \bm{z})
			 \log(1 - D(\bm{x}, \bm{z})) d\bm{x} d\bm{z}.
\end{split}
\end{equation}

An attractive property of adversarial approaches is that they do not require
that the conditional densities can be computed; they only require that they can
be sampled from in a way that allows gradient backpropagation. In the case of
ALI, this means that gradients should propagate from the discriminator network
to the encoder and decoder networks.

This can be done using the the reparametrization trick
\citep{kingma2013fast,bengio2013deep,bengio2013estimating}. Instead of sampling
directly from the desired distribution, the random variable is computed as a
deterministic transformation of some noise such that its distribution is the
desired distribution. For instance, if $q(z \mid x) = \mathcal{N}(\mu(x),
\sigma^2(x)I)$, one can draw samples by computing
\begin{equation}
    z = \mu(x) + \sigma(x) \odot \epsilon, \quad
    \epsilon \sim \mathcal{N}(0, I).
\end{equation}

More generally, one can employ a change of variable of the form
\begin{equation}
    v = f(u, \epsilon)
\end{equation}
where $\epsilon$ is some random source of noise.

The discriminator is trained to distinguish between samples from the encoder
$(\bm{x},\hat{\bm{z}}) \sim q(\bm{x}, \bm{z})$ and samples from the decoder
$(\tilde{\bm{x}}, \bm{z}) \sim p(\bm{x},\bm{z})$. The generator is trained to
fool the discriminator, i.e., to generate $\bm{x}, \bm{z}$ pairs from
$q(\bm{x},\bm{z})$ or $p(\bm{x}, \bm{z})$ that are indistinguishable one from
another. See \autoref{fig:model} for a diagram of the adversarial game and
\autoref{alg:ali} for an algorithmic description of the procedure.

In such a setting, and under the assumption of an optimal discriminator, the
generator minimizes the Jensen-Shannon divergence \citep{lin1991divergence}
between $q(\bm{x}, \bm{z})$ and $p(\bm{x}, \bm{z})$. This can be shown using
the same proof sketch as in the original GAN
paper~\citep{goodfellow2014generative}.

\begin{algorithm}[t]
\begin{algorithmic}
    \State $\theta_{g}, \theta_{d} \gets \text{initialize network parameters}$
    \Repeat
		\State $\bm{x}^{(1)}, \ldots, \bm{x}^{(M)} \sim q(\bm{x})$
            \Comment{Draw $M$ samples from the dataset and the prior}
		\State $\bm{z}^{(1)}, \ldots, \bm{z}^{(M)} \sim p(\bm{z})$
		\State $\hat{\bm{z}}^{(i)} \sim q(\bm{z} \mid \bm{x} = \bm{x}^{(i)}),
			    \quad i = 1, \ldots, M$
            \Comment{Sample from the conditionals}
		\State $\tilde{\bm{x}}^{(j)} \sim p(\bm{x} \mid \bm{z} = \bm{z}^{(j)}),
				\quad j = 1, \ldots, M$
        \State $\rho_q^{(i)} \gets D(\bm{x}^{(i)}, \hat{\bm{z}}^{(i)}),
				\quad i = 1, \ldots, M$
            \Comment{Compute discriminator predictions}
        \State $\rho_p^{(j)} \gets D(\tilde{\bm{x}}^{(j)}, \bm{z}^{(j)}),
				\quad j = 1, \ldots, M$
        \State $\mathcal{L}_d \gets
            -\frac{1}{M} \sum_{i=1}^M \log(\rho_q^{(i)})
            -\frac{1}{M} \sum_{j=1}^M\ log(1 - \rho_p^{(j)})$
            \Comment{Compute discriminator loss}
        \State $\mathcal{L}_g \gets
            -\frac{1}{M} \sum_{i=1}^M \log(1 - \rho_q^{(i)})
            -\frac{1}{M} \sum_{j=1}^M \log(\rho_p^{(j)})$
            \Comment{Compute generator loss}
        \State $\theta_d \gets \theta_d - \nabla_{\theta_d} \mathcal{L}_d$
            \Comment{Gradient update on discriminator network}
        \State $\theta_g \gets \theta_g - \nabla_{\theta_g} \mathcal{L}_g$
            \Comment{Gradient update on generator networks}
    \Until{convergence}
\end{algorithmic}
\caption{\label{alg:ali} The ALI training procedure.}
\end{algorithm}

\subsection{Relation to GAN}
ALI bears close resemblance to GAN, but it differs from it in the two following
ways:

\begin{itemize}
	\item The generator has two components: the encoder, $G_z(\bm{x})$, which
		maps data samples $\bm{x}$ to $\bm{z}$-space, and the decoder
		$G_x(\bm{z})$, which maps samples from the prior $p(\bm{z})$ (a source
		of noise) to the input space.
	\item The discriminator is trained to distinguish between joint pairs
		$(\bm{x}, \hat{\bm{z}} = G_x(\bm{x}))$ and $(\tilde{\bm{x}} =
		G_x(\bm{z}), \bm{z})$, as opposed to marginal samples $\bm{x} \sim
		q(\bm{x})$ and $\tilde{\bm{x}} \sim p(\bm{x})$.
\end{itemize}

\subsection{Alternative approaches to feedforward inference in GANs}
\label{sec:alternative}

The ALI training procedure is not the only way one could learn a feedforward
inference network in a GAN setting.

In recent work, \citet{chen2016infogan} introduce a model called InfoGAN which
minimizes the mutual information between a subset $\bm{c}$ of the latent code
and $\bm{x}$ through the use of an auxiliary distribution $Q(\bm{c} \mid
\bm{x})$. However, this does not correspond to full inference on $\bm{z}$, as
only the value for $\bm{c}$ is inferred. Additionally, InfoGAN requires that
$Q(\bf{c} \mid \bf{x})$ is a tractable approximate posterior that can be sampled
from and evaluated. ALI only requires that inference networks can be sampled
from, allowing it to represent arbitrarily complex posterior distributions.

One could learn the inverse mapping from GAN samples: this corresponds to
learning an encoder to reconstruct $\bm{z}$, i.e. finding an encoder such that
$\mathbb{E}_{z \sim p(z)}[\|z - G_z(G_x(z))\|_2^2] \approx 0$. We are not aware
of any work that reports results for this approach. This resembles the InfoGAN
learning procedure but with a fixed generative model and a factorial Gaussian
posterior with a fixed diagonal variance.

Alternatively, one could decompose training into two phases. In the first phase,
a GAN is trained normally. In the second phase, the GAN's decoder is frozen and
an encoder is trained following the ALI procedure (i.e., a discriminator taking
{\em both} $\bm{x}$ and $\bm{z}$ as input is introduced). We call this {\em
post-hoc learned inference}. In this setting, the encoder and the decoder cannot
interact together during training and the encoder must work with whatever the
decoder has learned during GAN training. Post-hoc learned inference may be
suboptimal if this interaction is beneficial to modeling the data distribution.

\subsection{Generator value function}

As with GANs, when ALI's discriminator gets too far ahead, its generator may
have a hard time minimizing the value function in \autoref{eq:value_function}.
If the discriminator's output is sigmoidal, then the gradient of the value
function with respect to the discriminator's output vanishes to zero as the
output saturates.

As a workaround, the generator is trained to maximize
\begin{equation}
\label{eq:generator_value_function}
    V'(D, G) = \mathbb{E}_{q(\bm{x})}[\log(1 - D(\bm{x}, G_{\bm{z}}(\bm{x})))] +
               \mathbb{E}_{p(\bm{z})}[\log(D(G_{\bm{x}}(\bm{z}), \bm{z}))] \\
\end{equation}
which has the same fixed points but whose gradient is stronger when the
discriminator's output saturates.

The adversarial game does not require an analytical expression for the joint
distributions. This means we can introduce variable changes without having to
know the explicit distribution over the new variable.  For instance, sampling
from $p(z)$ could be done by sampling $\epsilon \sim \mathcal{N}(0, I)$ and
passing it through an arbitrary differentiable function $z = f(\epsilon)$.

However, gradient propagation into the encoder and decoder networks relies on
the reparametrization trick, which means that ALI is not directly
applicable to either applications with discrete data or to models with
discrete latent variables.

\begin{figure}[p]
    \centering
    \begin{subfigure}[t]{0.49\textwidth}
        \centering
        \includegraphics[height=\figureHeight]{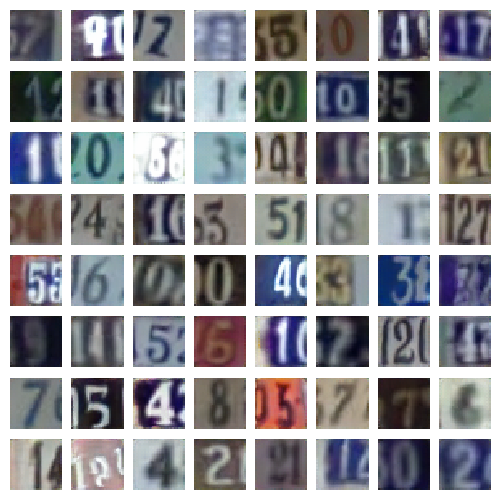}
        \caption{\label{fig:svhn_samples} SVHN samples.}
    \end{subfigure}
    \hfill
    \begin{subfigure}[t]{0.49\textwidth}
        \centering
        \includegraphics[height=\figureHeight]{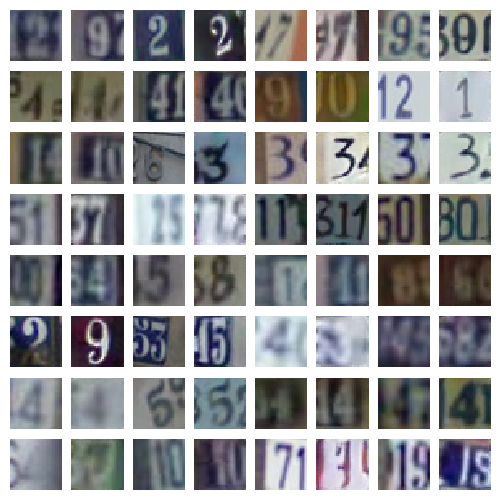}
        \caption{\label{fig:svhn_reconstructions} SVHN reconstructions.}
    \end{subfigure}
    \caption{\label{fig:svhn_images} Samples and reconstructions on the SVHN
        dataset. For the reconstructions, odd columns are
        original samples from the validation set and even columns are
        corresponding reconstructions (e.g., second column contains
        reconstructions of the first column's validation set samples).}
\end{figure}

\begin{figure}[p]
    \centering
    \begin{subfigure}[t]{0.49\textwidth}
        \centering
        \includegraphics[height=\figureHeight]{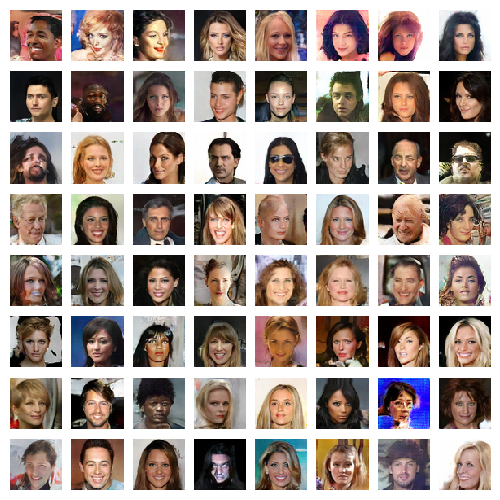}
        \caption{\label{fig:celeba_samples} CelebA samples.}
    \end{subfigure}
    \hfill
    \begin{subfigure}[t]{0.49\textwidth}
        \centering
        \includegraphics[height=\figureHeight]{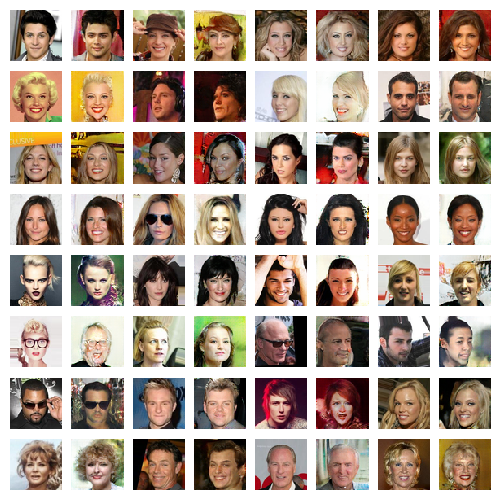}
        \caption{\label{fig:celeba_reconstructions} CelebA reconstructions.}
    \end{subfigure}
    \caption{\label{fig:celeba_images} Samples and reconstructions on the CelebA
        dataset. For the reconstructions, odd columns are
        original samples from the validation set and even columns are
        corresponding reconstructions.}
\end{figure}

\begin{figure}[p]
    \centering
    \begin{subfigure}[t]{0.49\textwidth}
        \centering
        \includegraphics[height=\figureHeight]{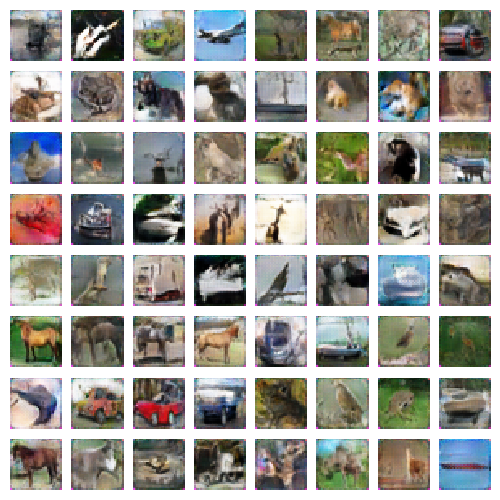}
        \caption{\label{fig:cifar10_samples} CIFAR10 samples.}
    \end{subfigure}
    \hfill
    \begin{subfigure}[t]{0.49\textwidth}
        \centering
        \includegraphics[height=\figureHeight]{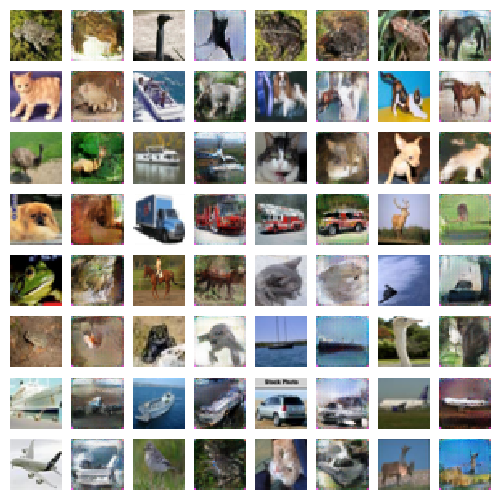}
        \caption{\label{fig:cifar10_reconstructions} CIFAR10
          reconstructions.}
    \end{subfigure}
    \caption{\label{fig:cifar10_images} Samples and reconstructions on the
        CIFAR10 dataset. For the reconstructions, odd columns are
        original samples from the validation set and even columns are
        corresponding reconstructions.}
\end{figure}

\subsection{Discriminator optimality}
\label{sec:disc_opt}
\begin{prop}
    Given a fixed generator $G$, the optimal discriminator is given by
    \begin{equation}
    \label{eq:optgansol}
		D^*(x, z) = \frac{q(x, z)}{q(x, z) + p(x, z)}.
    \end{equation}
\end{prop}
\begin{proof}
	For a fixed generator $G$, the complete data value function is
	\begin{equation}
	\label{eq:cdganvalue}
		V(D, G) = \EE_{x,z \sim q(x,z)}[\log(D(x, z))]
		        + \EE_{x, z \sim p(x, z)}[\log(1 - D(x, z))].
	\end{equation}
	The result follows by the concavity of the log and the simplified
	Euler-Lagrange equation first order conditions on
	$(x, z) \rightarrow D(x, z)$.
\end{proof}

\subsection{Relationship with the Jensen-Shannon divergence}
\label{sec:jsd_rel}
\begin{prop}
	Under an optimal discriminator $D^{*}$, the generator minimizes the
	Jensen-Shanon divergence which attains its minimum if and only if
	$q(\bm{x}, \bm{z}) = p(\bm{x}, \bm{z})$.
\end{prop}
\begin{proof}
	The proof is a straightforward extension of the proof
	in~\cite{goodfellow2014generative}.
\end{proof}

\subsection{Invertibility}
\label{sec:invertibility}
\begin{prop}
    Assuming optimal discriminator $D$ and generator $G$. If the encoder
    $G_{\bm{x}}$ is deterministic, then $G_{\bm{x}} = G^{-1}_{\bm{z}}$ and
    $G_{\bm{z}} = G^{-1}_{\bm{x}}$ almost everywhere.
\end{prop}
\begin{sketchproof}
    Consider the event $R_{\epsilon} = \{\bm{x} : \Vert x - (G_{\bm{x}} \circ
    G_{\bm{z}})(\bm{x})) \Vert > \epsilon\}$ for some positive $\epsilon$. This
    set can be seen as a section of the $(\bm{x}, \bm{z})$ space 	over the
    elements $\bm{z}$ such that $\bm{z} = G_{\bm{z}}(x)$. The generator being
    optimal, the probabilities of $R_{\epsilon}$ under $p(\bm{x}, \bm{z})$ and
    $q(\bm{x}, \bm{z})$ are equal.  Now $p(\bm{x} \mid \bm{z}) = \delta_{x -
    G_{x}(z)}$, where $\delta$ is the Dirac delta distribution.  This is enough
    to show that there are no $x$ satisfying the event $R_{\epsilon}$ and thus
    $G_{\bm{x}} = G^{-1}_{z}$ almost everywhere.  By symmetry, the same argument
    can be applied to show that $G_{\bm{z}} = G^{-1}_{\bm{x}}$. \\ The complete
    proof is given in \citep{donahue2016adversarial}, in which the authors
    independently examine the same model structure under the name Bidirectional
    GAN (BiGAN).
\end{sketchproof}

\section{Related Work}

Other recent papers explore hybrid approaches to generative modeling. One such
approach is to relax the probabilistic interpretation of the VAE model by
replacing either the KL-divergence term or the reconstruction term with
variants that have better properties. The adversarial autoencoder model
\citep{makhzani2015adversarial} replaces the KL-divergence term with a
discriminator that is trained to distinguish between approximate posterior and
prior samples, which provides a more flexible approach to matching the marginal
$q(\bm{z})$ and the prior. Other papers explore replacing the reconstruction
term with either GANs or auxiliary networks. \citet{larsen2015autoencoding}
collapse the decoder of a VAE and the generator of a GAN into one network in
order to supplement the reconstruction loss with a learned similarity metric.
\citet{lamb2016discriminative} use the hidden layers of a pre-trained
classifier as auxiliary reconstruction losses to help the VAE focus on
higher-level details when reconstructing. \citet{dosovitskiy2016generating}
combine both ideas into a unified loss function.

ALI's approach is also reminiscent of the adversarial autoencoder model, which
employs a GAN to distinguish between samples from the approximate posterior
distribution $q(\bm{z} \mid \bm{x})$ and prior samples. However, unlike
adversarial autoencoders, no explicit reconstruction loss is being optimized in
ALI, and the discriminator receives joint pairs of samples $(\bm{x}, \bm{z})$
rather than marginal $\bm{z}$ samples.

Independent work by \citet{donahue2016adversarial} proposes the same model under
the name Bidirectional GAN (BiGAN), in which the authors emphasize the learned
features' usefulness for auxiliary supervised and semi-supervised tasks. The
main difference in terms of experimental setting is that they use a
deterministic $q(\bm{z} \mid \bm{x})$ network, whereas we use a stochastic
network. In our experience, this does not make a big difference when $\bm{x}$ is
a deterministic function of $\bm{z}$ as the stochastic inference networks tend
to become determinstic as training progresses. When using stochastic mappings
from $\bm{z}$ to $\bm{x}$, the additional flexiblity of stochastic posteriors is
critical.

\section{Experimental results}

\begin{figure}[t]
    \centering
    \begin{subfigure}[t]{0.49\textwidth}
        \centering
        \includegraphics[height=\figureHeight]{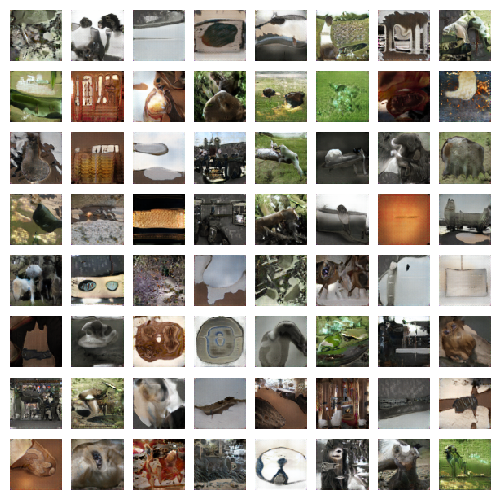}
        \caption{\label{fig:tiny_imagenet_samples} Tiny ImageNet samples.}
    \end{subfigure}
    \hfill
    \begin{subfigure}[t]{0.49\textwidth}
        \centering
        \includegraphics[height=\figureHeight]{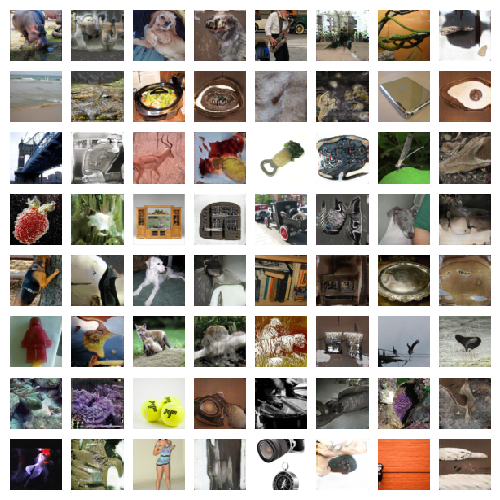}
        \caption{\label{fig:tiny_imagenet_reconstructions} Tiny ImageNet
            reconstructions.}
    \end{subfigure}
    \caption{\label{fig:tiny_imagenet_images} Samples and reconstructions on the
        Tiny ImageNet dataset. For the reconstructions, odd columns are original
        samples from the validation set and even columns are corresponding
        reconstructions.}
\end{figure}

We applied ALI to four different datasets, namely CIFAR10
\citep{krizhevsky2009learning}, SVHN \citep{netzer2011reading}, CelebA
\citep{liu2015deep} and a center-cropped, $64 \times 64$ version of the ImageNet
dataset \citep{russakovsky2015imagenet}.\footnote{
    The code for all experiments can be found at
    \url{https://github.com/IshmaelBelghazi/ALI}. Readers can also consult the
    accompanying website at \url{https://ishmaelbelghazi.github.io/ALI}.}

Transposed convolutions are used in $G_x(\bm{z})$. This operation corresponds
to the transpose of the matrix representation of a convolution, i.e., the
gradient of the convolution with respect to its inputs. For more details about
transposed convolutions and related operations,
see~\citet{dumoulin2016guide,shi2016deconvolution,odena2016deconvolution}.

\subsection{Samples and Reconstructions}
For each dataset, samples are presented (Figures \ref{fig:svhn_samples},
\ref{fig:celeba_samples} \ref{fig:cifar10_samples} and
\ref{fig:tiny_imagenet_samples}). They exhibit the same image fidelity as
samples from other adversarially-trained models.

We also qualitatively evaluate the fit between the conditional distribution
$q(\bm{z} \mid \bm{x})$ and the posterior distribution $p(\bm{z} \mid \bm{x})$
by sampling $\hat{\bm{z}} \sim q(\bm{z} \mid \bm{x})$ and $\hat{\bm{x}} \sim
p(\bm{x} \mid \bm{z} = \hat{\bm{z}})$ (Figures \ref{fig:svhn_reconstructions},
\ref{fig:celeba_reconstructions}, \ref{fig:cifar10_reconstructions} and
\ref{fig:tiny_imagenet_reconstructions}). This corresponds to reconstructing the
input in a VAE setting. Note that the ALI training objective does {\em not}
involve an explicit reconstruction loss.

We observe that reconstructions are not always faithful reproductions of the
inputs.  They retain the same crispness and quality characteristic to
adversarially-trained models, but oftentimes make mistakes in capturing exact
object placement, color, style and (in extreme cases) object identity. The
extent to which reconstructions deviate from the inputs varies between datasets:
on CIFAR10, which arguably constitutes a more complex input distribution, the
model exhibits less faithful reconstructions. This leads us to believe that poor
reconstructions are a sign of underfitting.

This failure mode represents an interesting departure from the bluriness
characteristic to the typical VAE setup. We conjecture that in the underfitting
regime, the latent variable representation learned by ALI is potentially more
invariant to less interesting factors of variation in the input and do not
devote model capacity to capturing these factors.

\subsection{Latent space interpolations}
As a sanity check for overfitting, we look at latent space interpolations
between validation set examples (\autoref{fig:celeba_interpolations}). We sample
pairs of validation set examples $\mathbf{x}_1$ and $\mathbf{x}_2$ and project
them into $\mathbf{z}_1$ and $\mathbf{z}_2$ by sampling from the encoder. We
then linearly interpolate between $\mathbf{z}_1$ and $\mathbf{z}_2$ and pass the
intermediary points through the decoder to plot the input-space interpolations.

We observe smooth transitions between pairs of examples, and intermediary images
remain believable. This is an indicator that ALI is not concentrating its
probability mass exclusively around training examples, but rather has learned
latent features that generalize well.

\begin{figure}[t]
    \centering
    \includegraphics[width=\textwidth]{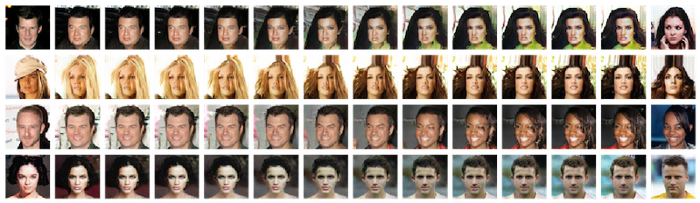}
    \caption{\label{fig:celeba_interpolations} Latent space interpolations on
        the CelebA validation set. Left and right columns correspond to the
        original pairs $\mathbf{x}_1$ and $\mathbf{x}_2$, and the columns in
        between correspond to the decoding of latent representations
        interpolated linearly from $\mathbf{z}_1$ to $\mathbf{z}_2$. Unlike other
        adversarial approaches like DCGAN~\citep{radford2015unsupervised}, ALI
        allows one to interpolate between actual data points.}
\end{figure}

\subsection{Semi-supervised learning}

We investigate the usefulness of the latent representation learned by ALI
through semi-supervised benchmarks on SVHN and CIFAR10.

We first compare with GAN on SVHN by following the procedure outlined in
\citet{radford2015unsupervised}. We train an L2-SVM on the learned
representations of a model trained on SVHN. The last three hidden layers of the
encoder as well as its output are concatenated to form a 8960-dimensional
feature vector. A 10,000 example held-out validation set is taken from the
training set and is used for model selection. The SVM is trained on 1000
examples taken at random from the remainder of the training set. The test error
rate is measured for 100 different SVMs trained on different random 1000-example
training sets, and the average error rate is measured along with its standard
deviation.

Using ALI's inference network as opposed to the discriminator to extract
features, we achieve a misclassification rate that is roughly $3.00 \pm 0.50\%$
lower than reported in \citet{radford2015unsupervised}
(\autoref{tab:semi-supervised-svhn}), which suggests that ALI's inference
mechanism is beneficial to the semi-supervised learning task.

We then investigate ALI's performance when label information is taken into
account during training. We adapt the discriminative model proposed in
\citet{Salimans2016gan}. The discriminator takes $x$ and $z$ as input and
outputs a distribution over $K + 1$ classes, where $K$ is the number of
categories. When label information is available for $q(x, z)$ samples, the
discriminator is expected to predict the label. When no label information is
available, the discriminator is expected to predict $K + 1$ for $p(x, z)$ samples
and $k \in \{1, \ldots, K\}$ for $q(x, z)$ samples.

Interestingly, \citet{Salimans2016gan} found that they required an alternative
training strategy for the generator where it tries to match first-order
statistics in the discriminator's intermediate activations with respect to the
data distribution (they refer to this as {\em feature matching}). We found that
ALI did not require feature matching to obtain comparable results. We achieve
results competitive with the state-of-the-art, as shown in Tables
\ref{tab:semi-supervised-svhn} and \ref{tab:semi-supervised-cifar10}.
\autoref{tab:semi-supervised-cifar10} shows that ALI offers a modest improvement
over \citet{Salimans2016gan}, more specifically for 1000 and 2000 labeled
examples.

\begin{table}[ht]
\caption{\label{tab:semi-supervised-svhn} SVHN test set missclassification rate}.
\centering
\scalebox{0.75}{
\begin{tabular}{@{}ll@{}} \toprule
Model & Misclassification rate \\ \midrule
VAE (M1 + M2) \citep{kingma2014semi}                    & $36.02$ \\
SWWAE with dropout \citep{zhao2015stacked}              & $23.56$ \\
DCGAN + L2-SVM \citep{radford2015unsupervised}          & $22.18$ \\
SDGM \citep{maaloe2016auxiliary}                        & $16.61$ \\ \midrule
\textbf{GAN (feature matching) \citep{Salimans2016gan}} & $\mathbf{8.11 \pm 1.3}$ \\
        ALI (ours, L2-SVM)                              & $19.14 \pm 0.50$ \\
\textbf{ALI (ours, no feature matching)}                & $\mathbf{7.42 \pm 0.65}$ \\
\bottomrule
\end{tabular}
}
\vspace{0.2cm}
\end{table}

\begin{table}[ht]
\caption{\label{tab:semi-supervised-cifar10} CIFAR10 test set missclassification
    rate for semi-supervised learning using different numbers of trained labeled
    examples. For ALI, error bars correspond to 3 times the standard deviation.}
\centering
\scalebox{0.75}{
\begin{tabular}{@{}lllll@{}} \toprule
Number of labeled examples & 1000 & 2000 & 4000 & 8000 \\
Model & \multicolumn{4}{c}{Misclassification rate} \\ \midrule
Ladder network \citep{ladder2015}              &              &      & $20.40$ &   \\
CatGAN \citep{catgan2015}                      &              &      & $19.58$ &   \\ \midrule
\textbf{GAN (feature matching) \citep{Salimans2016gan}} & $\mathbf{21.83 \pm 2.01}$ & $\mathbf{19.61 \pm 2.09}$ & $\mathbf{18.63 \pm 2.32}$ & $\mathbf{17.72 \pm 1.82}$ \\
\textbf{ALI (ours, no feature matching)} & $\mathbf{19.98 \pm 0.89}$ & $\mathbf{19.09 \pm 0.44}$ & $\mathbf{17.99 \pm 1.62}$ & $\mathbf{17.05 \pm 1.49}$ \\
\bottomrule
\end{tabular}
}
\vspace{0.2cm}
\end{table}

We are still investigating the differences between ALI and GAN with respect to
feature matching, but we conjecture that the latent representation learned by
ALI is better untangled with respect to the classification task and that it
generalizes better.

\subsection{Conditional Generation}
We extend ALI to match a conditional distribution. Let $\bm{y}$ represent a
fully observed conditioning variable. In this setting, the value function reads
\begin{equation}
\label{eq:conditional_value_function}
\min_G \max_D V(D, G) = \mathbb{E}_{q(\bm{x})\, p(\bm{y})}[\log(D(\bm{x}, G_z(\bm{x}, \bm{y}), \bm{y}))] + \mathbb{E}_{p(\bm{z})\, p(\bm{y})}[\log(1 - D(G_x(\bm{z}, \bm{y}), \bm{z}, \bm{y}))]
\end{equation}
We apply the conditional version of ALI to CelebA using the dataset's 40 binary
attributes. The attributes are linearly embedded in the encoder, decoder and
discriminator. We observe how a single element of the latent space $z$ changes
with respect to variations in the attributes vector $y$. Conditional samples are
shown in \autoref{fig:celeba_conditional_sequence}.

\begin{figure}[htb!]
    \centering
    \includegraphics[width=\textwidth]{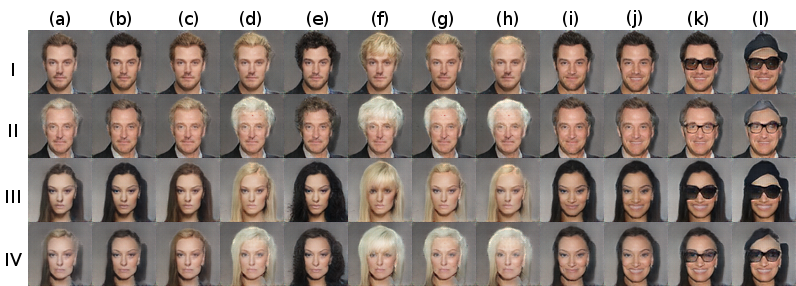}
    \caption{\label{fig:celeba_conditional_sequence} Conditional generation
    sequence. We sample a single fixed latent code $z$. Each row has a subset of
    attributes that are held constant across columns. The attributes are male,
    attractive, young for row I; male, attractive, older for row II; female,
    attractive, young for row III; female, attractive, older for Row IV.
    Attributes are then varied uniformly over rows across all columns in the
    following sequence: (b) black hair; (c) brown hair; (d) blond hair; (e)
    black hair, wavy hair; (f) blond hair, bangs; (g) blond hair, receding
    hairline; (h) blond hair, balding; (i) black hair, smiling; (j) black hair,
    smiling, mouth slightly open; (k) black hair, smiling, mouth slightly open,
    eyeglasses; (l) black hair, smiling, mouth slightly open, eyeglasses,
    wearing hat. }
\end{figure}

\subsection{Importance of learning inference jointly with generation}

\begin{figure}[htb!]
    \centering
    \includegraphics[width=\textwidth]{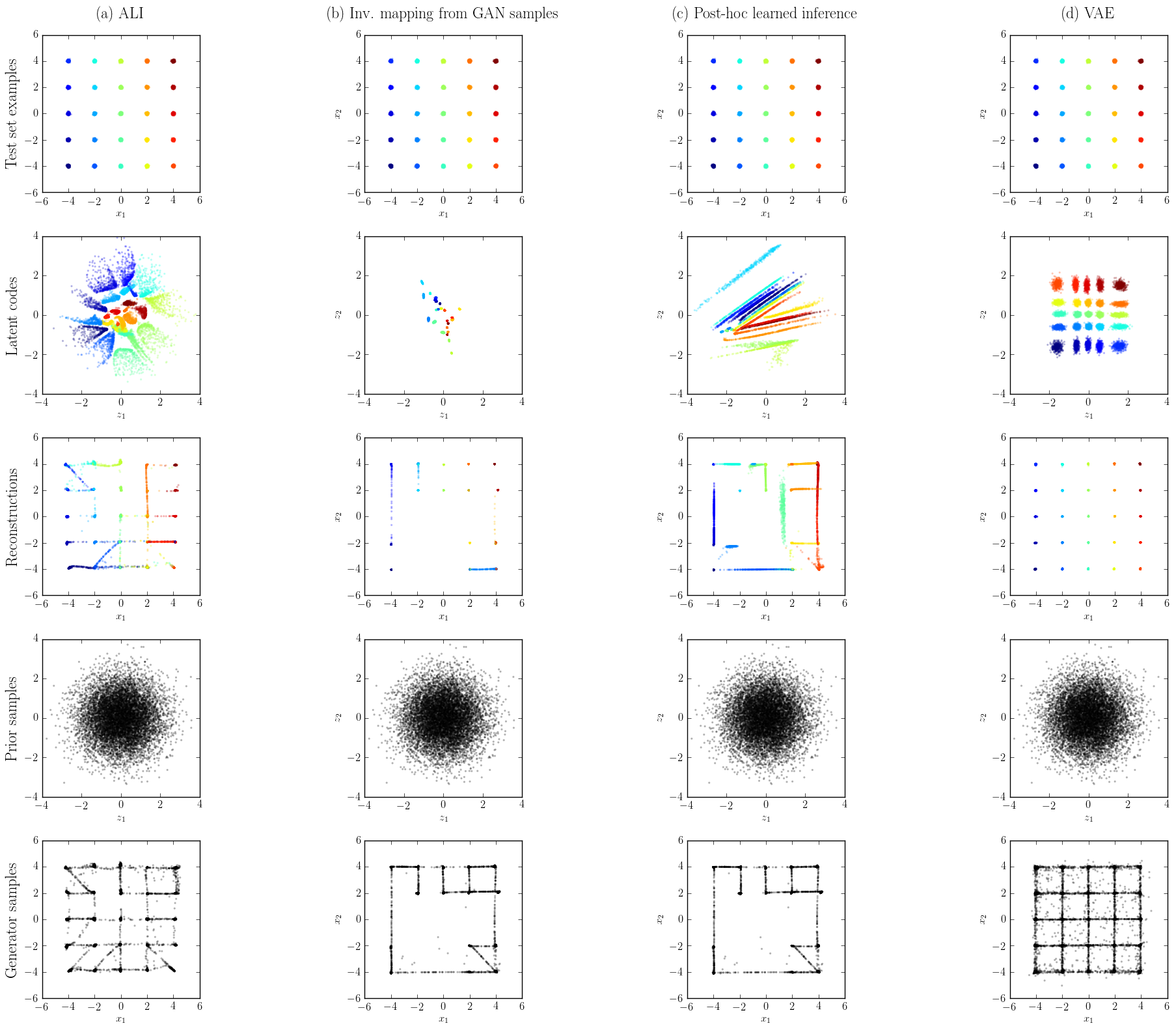}
    \caption{\label{fig:mixture_plot} Comparison of (a) ALI, (b) GAN with an
    encoder learned to reconstruct latent samples (c) GAN with an encoder
    learned through ALI, (d) variational autoencoder (VAE) on a 2D toy dataset.
    The ALI model in (a) does a much better job of covering the latent space
    (second row) and producing good samples than the two GAN models (b, c)
    augmented with an inference mechanism.}
\end{figure}

To highlight the role of the inference network during learning, we performed an
experiment on a toy dataset for which $q(\bm{x})$ is a 2D gaussian mixture with
25 mixture components laid out on a grid. The covariance matrices and centroids
have been chosen such that the distribution exhibits lots of modes separated by
large low-probability regions, which makes it a decently hard task despite the
2D nature of the dataset.

We trained ALI and GAN on 100,000 $q(\bm{x})$ samples. The decoder and
discriminator architectures are identical between ALI and GAN (except for the
input of the discriminator, which receives the concatenation of $\bf{x}$ and
$\bf{z}$ in the ALI case). Each model was trained 10 times using Adam
\citep{kingma2014adam} with random learning rate and $\beta_1$ values, and the
weights were initialized by drawing from a gaussian distribution with a random
standard deviation.

We measured the extent to which the trained models covered all 25 modes by
drawing 10,000 samples from their $p(\bm{x})$ distribution and assigning each
sample to a $q(\bm{x}$) mixture component according to the mixture
responsibilities. We defined a dropped mode as one that wasn’t assigned to {\em
any} sample. Using this definition, we found that ALI models covered $13.4 \pm
5.8$ modes on average (min: 8, max: 25) while GAN models covered $10.4 \pm 9.2$
modes on average (min: 1, max: 22).

We then selected the best-covering ALI and GAN models, and the GAN model was
augmented with an encoder using the {\em learned inverse mapping} and {\em
post-hoc learned inference} procedures outlined in \autoref{sec:alternative}.
The encoders learned for GAN inference have the same architecture as ALI’s
encoder.  We also trained a VAE with the same encoder-decoder architecture as
ALI to outline the qualitative differences between ALI and VAE models.

We then compared each model’s inference capabilities by reconstructing 10,000
held-out samples from $q(\bm{x})$. \autoref{fig:mixture_plot} summarizes the
experiment. We observe the following:
\begin{itemize}
    \item The ALI encoder models a marginal distribution $q(\bm{z})$ that
        matches $\bm{p(z)}$ fairly well (row 2, column a). The learned
        representation does a decent job at clustering and organizing the
        different mixture components.
    \item The GAN generator (row 5, columns b-c) has more trouble reaching all
        the modes than the ALI generator (row 5, column a), even over 10 runs of
        hyperparameter search.
    \item Learning an inverse mapping from GAN samples does not work very well:
        the encoder has trouble covering the prior marginally and the way it
        clusters mixture components is not very well organized (row 2, column
        b). As discussed in \autoref{sec:alternative}, reconstructions suffer
        from the generator dropping modes.
    \item Learning inference post-hoc doesn't work as well as training the
        encoder and the decoder jointly. As had been hinted at in
        \autoref{sec:alternative}, it appears that adversarial training benefits
        from learning inference at training time in terms of mode coverage. This
        also negatively impacts how the latent space is organized (row 2, column
        c). However, it appears to be better at matching $q(\bm{z})$ and
        $p(\bm{z})$ than when inference is learned through inverse mapping from
        GAN samples.
    \item Due to the nature of the loss function being optimized, the VAE model
        covers all modes easily (row 5, column d) and excels at reconstructing
        data samples (row 3, column d). However, they have a much more
        pronounced tendency to smear out their probability density (row 5,
        column d) and leave ``holes'' in $\bf{q(z)}$ (row 2, column d). Note
        however that recent approaches such as Inverse Autoregressive Flow
        \citep{kingma2016improving} may be used to improve on this, at the cost
        of a more complex mathematical framework.
\end{itemize}

In summary, this experiment provides evidence that adversarial training benefits
from learning an inference mechanism jointly with the decoder.  Furthermore, it
shows that our proposed approach for learning inference in an adversarial
setting is superior to the other approaches investigated.

\section{Conclusion}

We introduced the adversarially learned inference (ALI) model, which jointly
learns a generation network and an inference network using an adversarial
process. The model learns mutually coherent inference and generation networks,
as exhibited by its reconstructions. The induced latent variable mapping is
shown to be useful, achieving results competitive with the state-of-the-art
on the semi-supervised SVHN and CIFAR10 tasks.

\subsubsection*{Acknowledgments}

The authors would like to acknowledge the support of the following agencies for
research funding and computing support: NSERC, Calcul Qu\'{e}bec, Compute
Canada. We would also like to thank the developers of Theano
\citep{bergstra2010theano,bastien2012theano,theano2016theano}, Blocks and Fuel
\citep{van2015blocks}, which were used extensively for the paper. Finally, we
would like to thank Yoshua Bengio, David Warde-Farley, Yaroslav Ganin and
Laurent Dinh for their valuable feedback.

\bibliography{bibliography}
\bibliographystyle{iclr2017_conference}

\clearpage

\appendix
\section{Hyperparameters}

\begin{table}[h]
\centering
\begin{tabular}{@{}rllllll@{}} \toprule
Operation              & Kernel       & Strides      & Feature maps & BN?      & Dropout & Nonlinearity \\ \midrule
$G_z(x)$ -- $3 \times 32 \times 32$ input                                                             \\
Convolution            & $5 \times 5$ & $1 \times 1$ & $32$         & $\surd$  & 0.0     & Leaky ReLU \\
Convolution            & $4 \times 4$ & $2 \times 2$ & $64$         & $\surd$  & 0.0     & Leaky ReLU \\
Convolution            & $4 \times 4$ & $1 \times 1$ & $128$        & $\surd$  & 0.0     & Leaky ReLU \\
Convolution            & $4 \times 4$ & $2 \times 2$ & $256$        & $\surd$  & 0.0     & Leaky ReLU \\
Convolution            & $4 \times 4$ & $1 \times 1$ & $512$        & $\surd$  & 0.0     & Leaky ReLU \\
Convolution            & $1 \times 1$ & $1 \times 1$ & $512$        & $\surd$  & 0.0     & Leaky ReLU \\
Convolution            & $1 \times 1$ & $1 \times 1$ & $128$        & $\times$ & 0.0     & Linear     \\
$G_x(z)$ -- $64 \times 1 \times 1$ input                                                              \\
Transposed convolution & $4 \times 4$ & $1 \times 1$ & $256$        & $\surd$  & 0.0     & Leaky ReLU \\
Transposed convolution & $4 \times 4$ & $2 \times 2$ & $128$        & $\surd$  & 0.0     & Leaky ReLU \\
Transposed convolution & $4 \times 4$ & $1 \times 1$ & $64$         & $\surd$  & 0.0     & Leaky ReLU \\
Transposed convolution & $4 \times 4$ & $2 \times 2$ & $32$         & $\surd$  & 0.0     & Leaky ReLU \\
Transposed convolution & $5 \times 5$ & $1 \times 1$ & $32$         & $\surd$  & 0.0     & Leaky ReLU \\
Convolution            & $1 \times 1$ & $1 \times 1$ & $32$         & $\surd$  & 0.0     & Leaky ReLU \\
Convolution            & $1 \times 1$ & $1 \times 1$ & $3$          & $\times$ & 0.0     & Sigmoid    \\
$D(x)$ -- $3 \times 32 \times 32$ input                                                               \\
Convolution            & $5 \times 5$ & $1 \times 1$ & $32$         & $\times$ & 0.2     & Maxout     \\
Convolution            & $4 \times 4$ & $2 \times 2$ & $64$         & $\times$ & 0.5     & Maxout     \\
Convolution            & $4 \times 4$ & $1 \times 1$ & $128$        & $\times$ & 0.5     & Maxout     \\
Convolution            & $4 \times 4$ & $2 \times 2$ & $256$        & $\times$ & 0.5     & Maxout     \\
Convolution            & $4 \times 4$ & $1 \times 1$ & $512$        & $\times$ & 0.5     & Maxout     \\
$D(z)$ -- $64 \times 1 \times 1$ input                                                                \\
Convolution            & $1 \times 1$ & $1 \times 1$ & $512$        & $\times$ & 0.2     & Maxout     \\
Convolution            & $1 \times 1$ & $1 \times 1$ & $512$        & $\times$ & 0.5     & Maxout     \\
$D(x, z)$ -- $1024 \times 1 \times 1$ input                                                           \\
\multicolumn{7}{@{}c@{}}{\em Concatenate $D(x)$ and $D(z)$ along the channel axis}                    \\
Convolution            & $1 \times 1$ & $1 \times 1$ & $1024$       & $\times$ & 0.5     & Maxout     \\
Convolution            & $1 \times 1$ & $1 \times 1$ & $1024$       & $\times$ & 0.5     & Maxout     \\
Convolution            & $1 \times 1$ & $1 \times 1$ & $1$          & $\times$ & 0.5     & Sigmoid    \\ \midrule
Optimizer              & \multicolumn{6}{@{}l@{}}{Adam ($\alpha = 10^{-4}$, $\beta_1 = 0.5$, $\beta_2 = 10^{-3}$)} \\
Batch size             & \multicolumn{6}{@{}l@{}}{100}												  \\
Epochs                 & \multicolumn{6}{@{}l@{}}{6475}  											  \\
Leaky ReLU slope, maxout pieces       & \multicolumn{6}{@{}l@{}}{0.1, 2}                                                \\
Weight, bias initialization  & \multicolumn{6}{@{}l@{}}{Isotropic gaussian ($\mu = 0$, $\sigma = 0.01$), Constant($0$)} \\ \bottomrule
\end{tabular}
\vspace{0.2cm}
\caption{\label{tab:cifar10_description} CIFAR10 model hyperparameters (unsupervised). Maxout
    layers \citep{goodfellow2013maxout} are used in the discriminator.}
\end{table}

\begin{table}[h]
\centering
\begin{tabular}{@{}rllllll@{}} \toprule
Operation              & Kernel       & Strides      & Feature maps & BN?          & Dropout & Nonlinearity \\ \midrule
$G_z(x)$ -- $3 \times 32 \times 32$ input                                                                 \\
Convolution            & $5 \times 5$ & $1 \times 1$ & $32$         & $\surd$      & 0.0     & Leaky ReLU \\
Convolution            & $4 \times 4$ & $2 \times 2$ & $64$         & $\surd$      & 0.0     & Leaky ReLU \\
Convolution            & $4 \times 4$ & $1 \times 1$ & $128$        & $\surd$      & 0.0     & Leaky ReLU \\
Convolution            & $4 \times 4$ & $2 \times 2$ & $256$        & $\surd$      & 0.0     & Leaky ReLU \\
Convolution            & $4 \times 4$ & $1 \times 1$ & $512$        & $\surd$      & 0.0     & Leaky ReLU \\
Convolution            & $1 \times 1$ & $1 \times 1$ & $512$        & $\surd$      & 0.0     & Leaky ReLU \\
Convolution            & $1 \times 1$ & $1 \times 1$ & $512$        & $\times$     & 0.0     & Linear     \\
$G_x(z)$ -- $256 \times 1 \times 1$ input                                                                 \\
Transposed convolution & $4 \times 4$ & $1 \times 1$ & $256$        & $\surd$      & 0.0     & Leaky ReLU \\
Transposed convolution & $4 \times 4$ & $2 \times 2$ & $128$        & $\surd$      & 0.0     & Leaky ReLU \\
Transposed convolution & $4 \times 4$ & $1 \times 1$ & $64$         & $\surd$      & 0.0     & Leaky ReLU \\
Transposed convolution & $4 \times 4$ & $2 \times 2$ & $32$         & $\surd$      & 0.0     & Leaky ReLU \\
Transposed convolution & $5 \times 5$ & $1 \times 1$ & $32$         & $\surd$      & 0.0     & Leaky ReLU \\
Convolution            & $1 \times 1$ & $1 \times 1$ & $32$         & $\surd$      & 0.0     & Leaky ReLU \\
Convolution            & $1 \times 1$ & $1 \times 1$ & $3$          & $\times$     & 0.0     & Sigmoid    \\
$D(x)$ -- $3 \times 32 \times 32$ input                                                                   \\
Convolution            & $5 \times 5$ & $1 \times 1$ & $32$         & $\times$     & 0.2     & Leaky ReLU \\
Convolution            & $4 \times 4$ & $2 \times 2$ & $64$         & $\surd$      & 0.2     & Leaky ReLU \\
Convolution            & $4 \times 4$ & $1 \times 1$ & $128$        & $\surd$      & 0.2     & Leaky ReLU \\
Convolution            & $4 \times 4$ & $2 \times 2$ & $256$        & $\surd$      & 0.2     & Leaky ReLU \\
Convolution            & $4 \times 4$ & $1 \times 1$ & $512$        & $\surd$      & 0.2     & Leaky ReLU \\
$D(z)$ -- $256 \times 1 \times 1$ input                                                                   \\
Convolution            & $1 \times 1$ & $1 \times 1$ & $512$        & $\times$     & 0.2     & Leaky ReLU \\
Convolution            & $1 \times 1$ & $1 \times 1$ & $512$        & $\times$     & 0.2     & Leaky ReLU \\
$D(x, z)$ -- $1024 \times 1 \times 1$ input                                                               \\
\multicolumn{7}{@{}c@{}}{\em Concatenate $D(x)$ and $D(z)$ along the channel axis}                        \\
Convolution            & $1 \times 1$ & $1 \times 1$ & $1024$       & $\times$     & 0.2     & Leaky ReLU \\
Convolution            & $1 \times 1$ & $1 \times 1$ & $1024$       & $\times$     & 0.2     & Leaky ReLU \\
Convolution            & $1 \times 1$ & $1 \times 1$ & $1$          & $\times$     & 0.2     & Sigmoid    \\ \midrule
Optimizer              & \multicolumn{6}{@{}l@{}}{Adam ($\alpha = 10^{-4}$, $\beta_1 = 0.5$, $\beta_2 = 10^{-3}$)}  \\
Batch size             & \multicolumn{6}{@{}l@{}}{100}												      \\
Epochs                 & \multicolumn{6}{@{}l@{}}{100}												      \\
Leaky ReLU slope       & \multicolumn{6}{@{}l@{}}{0.01}                                                   \\
Weight, bias initialization  & \multicolumn{6}{@{}l@{}}{Isotropic gaussian ($\mu = 0$, $\sigma = 0.01$), Constant($0$)} \\ \bottomrule
\end{tabular}
\vspace{0.2cm}
\caption{\label{tab:svhn_description} SVHN model hyperparameters (unsupervised).}
\end{table}

\begin{table}[h]
\centering
\begin{tabular}{@{}rllllll@{}} \toprule
Operation              & Kernel       & Strides      & Feature maps & BN?          & Dropout & Nonlinearity \\ \midrule
$G_z(x)$ -- $3 \times 64 \times 64$ input                                                                 \\
Convolution            & $2 \times 2$ & $1 \times 1$ & $64$         & $\surd$      & 0.0     & Leaky ReLU \\
Convolution            & $7 \times 7$ & $2 \times 2$ & $128$        & $\surd$      & 0.0     & Leaky ReLU \\
Convolution            & $5 \times 5$ & $2 \times 2$ & $256$        & $\surd$      & 0.0     & Leaky ReLU \\
Convolution            & $7 \times 7$ & $2 \times 2$ & $256$        & $\surd$      & 0.0     & Leaky ReLU \\
Convolution            & $4 \times 4$ & $1 \times 1$ & $512$        & $\surd$      & 0.0     & Leaky ReLU \\
Convolution            & $1 \times 1$ & $1 \times 1$ & $512$        & $\times$     & 0.0     & Linear     \\
$G_x(z)$ -- $512 \times 1 \times 1$ input                                                                 \\
Transposed convolution & $4 \times 4$ & $1 \times 1$ & $512$        & $\surd$      & 0.0     & Leaky ReLU \\
Transposed convolution & $7 \times 7$ & $2 \times 2$ & $256$        & $\surd$      & 0.0     & Leaky ReLU \\
Transposed convolution & $5 \times 5$ & $2 \times 2$ & $256$         & $\surd$      & 0.0     & Leaky ReLU \\
Transposed convolution & $7 \times 7$ & $2 \times 2$ & $128$         & $\surd$      & 0.0     & Leaky ReLU \\
Transposed convolution & $2 \times 2$ & $1 \times 1$ & $64$         & $\surd$      & 0.0     & Leaky ReLU \\
Convolution            & $1 \times 1$ & $1 \times 1$ & $3$          & $\times$     & 0.0     & Sigmoid    \\
$D(x)$ -- $3 \times 64 \times 64$ input                                                                   \\
Convolution            & $2 \times 2$ & $1 \times 1$ & $64$         & $\surd$      & 0.0     & Leaky ReLU \\
Convolution            & $7 \times 7$ & $2 \times 2$ & $128$        & $\surd$      & 0.0     & Leaky ReLU \\
Convolution            & $5 \times 5$ & $2 \times 2$ & $256$        & $\surd$      & 0.0     & Leaky ReLU \\
Convolution            & $7 \times 7$ & $2 \times 2$ & $256$        & $\surd$      & 0.0     & Leaky ReLU \\
Convolution            & $4 \times 4$ & $1 \times 1$ & $512$        & $\surd$      & 0.0     & Leaky ReLU \\
$D(z)$ -- $512 \times 1 \times 1$ input                                                                   \\
Convolution            & $1 \times 1$ & $1 \times 1$ & $1024$        & $\times$     & 0.2     & Leaky ReLU \\
Convolution            & $1 \times 1$ & $1 \times 1$ & $1024$        & $\times$     & 0.2     & Leaky ReLU \\
$D(x, z)$ -- $1024 \times 1 \times 1$ input                                                               \\
\multicolumn{7}{@{}c@{}}{\em Concatenate $D(x)$ and $D(z)$ along the channel axis}                        \\
Convolution            & $1 \times 1$ & $1 \times 1$ & $2048$       & $\times$     & 0.2     & Leaky ReLU \\
Convolution            & $1 \times 1$ & $1 \times 1$ & $2048$       & $\times$     & 0.2     & Leaky ReLU \\
Convolution            & $1 \times 1$ & $1 \times 1$ & $1$          & $\times$     & 0.2     & Sigmoid    \\ \midrule
Optimizer              & \multicolumn{6}{@{}l@{}}{Adam ($\alpha = 10^{-4}$, $\beta_1 = 0.5$)}  \\
Batch size             & \multicolumn{6}{@{}l@{}}{100}												      \\
Epochs                 & \multicolumn{6}{@{}l@{}}{123}												      \\
Leaky ReLU slope       & \multicolumn{6}{@{}l@{}}{0.02}                                                   \\
Weight, bias initialization  & \multicolumn{6}{@{}l@{}}{Isotropic gaussian ($\mu = 0$, $\sigma = 0.01$), Constant($0$)} \\ \bottomrule
\end{tabular}
\vspace{0.2cm}
\caption{\label{tab:celeba_description} CelebA model hyperparameters (unsupervised).}
\end{table}

\begin{table}[h]
\centering
\begin{tabular}{@{}rllllll@{}} \toprule
Operation              & Kernel       & Strides      & Feature maps & BN?          & Dropout & Nonlinearity \\ \midrule
$G_z(x)$ -- $3 \times 64 \times 64$ input                                                                 \\
Convolution            & $4 \times 4$ & $2 \times 2$ & $64$         & $\surd$      & 0.0     & Leaky ReLU \\
Convolution            & $4 \times 4$ & $1 \times 1$ & $64$         & $\surd$      & 0.0     & Leaky ReLU \\
Convolution            & $4 \times 4$ & $2 \times 2$ & $128$        & $\surd$      & 0.0     & Leaky ReLU \\
Convolution            & $4 \times 4$ & $1 \times 1$ & $128$        & $\surd$      & 0.0     & Leaky ReLU \\
Convolution            & $4 \times 4$ & $2 \times 2$ & $256$        & $\surd$      & 0.0     & Leaky ReLU \\
Convolution            & $4 \times 4$ & $1 \times 1$ & $256$        & $\surd$      & 0.0     & Leaky ReLU \\
Convolution            & $1 \times 1$ & $1 \times 1$ & $2048$       & $\surd$      & 0.0     & Leaky ReLU \\
Convolution            & $1 \times 1$ & $1 \times 1$ & $2048$       & $\surd$      & 0.0     & Leaky ReLU \\
Convolution            & $1 \times 1$ & $1 \times 1$ & $512$        & $\times$     & 0.0     & Linear     \\
$G_x(z)$ -- $256 \times 1 \times 1$ input                                                                 \\
Convolution            & $1 \times 1$ & $1 \times 1$ & $2048$       & $\surd$      & 0.0     & Leaky ReLU \\
Convolution            & $1 \times 1$ & $1 \times 1$ & $256$        & $\surd$      & 0.0     & Leaky ReLU \\
Transposed convolution & $4 \times 4$ & $1 \times 1$ & $256$        & $\surd$      & 0.0     & Leaky ReLU \\
Transposed convolution & $4 \times 4$ & $2 \times 2$ & $128$        & $\surd$      & 0.0     & Leaky ReLU \\
Transposed convolution & $4 \times 4$ & $1 \times 1$ & $128$        & $\surd$      & 0.0     & Leaky ReLU \\
Transposed convolution & $4 \times 4$ & $2 \times 2$ & $64$         & $\surd$      & 0.0     & Leaky ReLU \\
Transposed convolution & $4 \times 4$ & $1 \times 1$ & $64$         & $\surd$      & 0.0     & Leaky ReLU \\
Transposed convolution & $4 \times 4$ & $2 \times 2$ & $64$         & $\surd$      & 0.0     & Leaky ReLU \\
Convolution            & $1 \times 1$ & $1 \times 1$ & $3$          & $\times$     & 0.0     & Sigmoid    \\
$D(x)$ -- $3 \times 64 \times 64$ input                                                                   \\
Convolution            & $4 \times 4$ & $2 \times 2$ & $64$         & $\times$     & 0.2     & Leaky ReLU \\
Convolution            & $4 \times 4$ & $1 \times 1$ & $64$         & $\surd$      & 0.2     & Leaky ReLU \\
Convolution            & $4 \times 4$ & $2 \times 2$ & $128$        & $\surd$      & 0.2     & Leaky ReLU \\
Convolution            & $4 \times 4$ & $1 \times 1$ & $128$        & $\surd$      & 0.2     & Leaky ReLU \\
Convolution            & $4 \times 4$ & $2 \times 2$ & $256$        & $\surd$      & 0.2     & Leaky ReLU \\
Convolution            & $4 \times 4$ & $1 \times 1$ & $256$        & $\surd$      & 0.2     & Leaky ReLU \\
$D(z)$ -- $256 \times 1 \times 1$ input                                                                   \\
Convolution            & $1 \times 1$ & $1 \times 1$ & $2048$       & $\times$     & 0.2     & Leaky ReLU \\
Convolution            & $1 \times 1$ & $1 \times 1$ & $2048$       & $\times$     & 0.2     & Leaky ReLU \\
$D(x, z)$ -- $2304 \times 1 \times 1$ input                                                               \\
\multicolumn{7}{@{}c@{}}{\em Concatenate $D(x)$ and $D(z)$ along the channel axis}                        \\
Convolution            & $1 \times 1$ & $1 \times 1$ & $4096$       & $\times$     & 0.2     & Leaky ReLU \\
Convolution            & $1 \times 1$ & $1 \times 1$ & $4096$       & $\times$     & 0.2     & Leaky ReLU \\
Convolution            & $1 \times 1$ & $1 \times 1$ & $1$          & $\times$     & 0.2     & Sigmoid    \\ \midrule
Optimizer              & \multicolumn{6}{@{}l@{}}{Adam ($\alpha = 10^{-4}$, $\beta_1 = 0.5$, $\beta_2 = 10^{-3}$)}  \\
Batch size             & \multicolumn{6}{@{}l@{}}{128}												      \\
Epochs                 & \multicolumn{6}{@{}l@{}}{125}												      \\
Leaky ReLU slope       & \multicolumn{6}{@{}l@{}}{0.01}                                                   \\
Weight, bias initialization  & \multicolumn{6}{@{}l@{}}{Isotropic gaussian ($\mu = 0$, $\sigma = 0.01$), Constant($0$)} \\ \bottomrule
\end{tabular}
\vspace{0.2cm}
\caption{\label{tab:tiny_imagenet_description} Tiny ImageNet model hyperparameters (unsupervised).}
\end{table}

\clearpage
\section{A generative story for ALI}

\begin{figure}[h]
    \centering
    \begin{tikzpicture}[remember picture,node distance=2cm,
                        box/.style={draw,rectangle,rounded corners}]
        \node[box,rectangle] (q) {
            \begin{tikzpicture}
                \node (x) {\small Xavier's painting};
                \node[above=of x] (z_hat) {Zelda's description};
            \end{tikzpicture}
        };
        \node[box,minimum height=1cm,minimum width=2cm,right=of q]
            (discriminator) {Mr. Discriminator};
        \node[box,right= of discriminator] (p) {
            \begin{tikzpicture}
                \node (x_tilde) {Xena's depiction};
                \node[above=of x_tilde] (z) {Zach's description};
            \end{tikzpicture}
        };
        \draw[->] (x) -- (z_hat);
        \draw[->] (z) -- (x_tilde);
        \draw[->] (q) -- (discriminator);
        \draw[->] (p) -- (discriminator);
    \end{tikzpicture}
    \caption{\label{fig:model_story} A Circle of Infinite Painters' view of the
        ALI game.}
\end{figure}
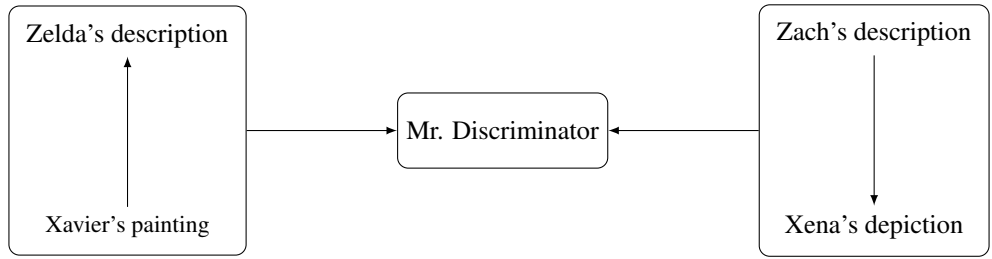

The Circle of Infinite Painters is a very prolific artistic group. Very little
is known about the Circle, but what we do know is that it is composed of two
very brilliant artists. It has produced new paintings almost daily for more than
twenty years, each one more beautiful than the others. Not only are the
paintings exquisite, but their title and description is by itself a literary
masterpiece.

However, some scholars believe that things might not be as they appear: certain
discrepancies in the Circle's body of work hints at the Circle being composed of
more than one artistic duo. This is what Joseph Discriminator, art critique and
world expert on the Circle, believes. He's recently been working intensively on
the subject. Without knowing it, he's right: the Circle is not one, but two
artistic duos.

Xavier and Zach Prior form the creative component of the group. Xavier is a
painter and can, in one hour and starting from nothing, produce a painting that
would make any great painter jealous. Impossible however for him to explain
what he's done: he works by intuition alone. Zach is an author and his literary
talent equals Xavier's artistic talent. His verb is such that the scenes he
describes could just as well be real.

By themselves, the Prior brothers cannot collaborate: Xavier can't paint
anything from a description and Zach is bored to death with the idea of
describing anything that does not come out of his head. This is why the Prior
brothers depend on the Conditional sisters so much.

Zelda Conditional has an innate descriptive talent: she can examine a painting
and describe it so well that the original would seem like an imitation. Xena
Conditional has a technical mastery of painting that allows her to recreate
everything that's described to her in the most minute details. However, their
creativity is inversely proportional to their talent: by themselves, they cannot
produce anything of interest.

As such, the four members of the Circle work in pairs. What Xavier paints, Zelda
describes, and what Zach describes, Xena paints. They all work together to
fulfill the same vision of a unified Circle of Infinite Painters, a whole greater
than the sum of its parts.

This is why Joseph Discriminator's observations bother them so much. Secretly,
the Circle put Mr. Discriminator under surveillance. Whatever new observation
he's made, they know right away and work on attenuating the differences to
maintain the illusion of a Circle of Infinite Painters made of a single artistic
duo.

Will the Circle reach this ideal, or will it be unmasked by Mr. Discriminator?

\end{document}